\documentclass[11pt,a4paper]{article}

\usepackage{jacad-onecol-arxiv}

\usepackage{amsmath,amssymb,amsthm}
\usepackage{algorithm}
\usepackage{algpseudocode}
\usepackage{float}
\usepackage{booktabs}
\usepackage{longtable}
\usepackage{array}
\usepackage{enumitem}
\usepackage{pdfpages}
\usepackage[utf8]{inputenc}
\usepackage[T1]{fontenc}
\usepackage{newunicodechar}

\newunicodechar{Ḥ}{\d{H}}
\newunicodechar{ḥ}{\d{h}}
\newunicodechar{Ṭ}{\d{T}}
\newunicodechar{ṭ}{\d{t}}
\newunicodechar{Ṣ}{\d{S}}
\newunicodechar{ṣ}{\d{s}}
\newunicodechar{Ā}{\={A}}
\newunicodechar{ā}{\={a}}
\newunicodechar{Ī}{\={I}}
\newunicodechar{ī}{\={i}}
\newunicodechar{Ū}{\={U}}
\newunicodechar{ū}{\={u}}
\newunicodechar{ʿ}{\textquoteleft}

\newtheorem{theorem}{Theorem}
\newtheorem{axiom}[theorem]{Axiom}
\newtheorem{corollary}[theorem]{Corollary}
\newtheorem{definition}[theorem]{Definition}
\newtheorem{remark}[theorem]{Remark}
\newtheorem{principle}[theorem]{Principle}
\newtheorem{heuristic}[theorem]{Heuristic Principle}

\arxivjournalref{J.Acad. (N.Y.)15,1:3--18}
\arxivyear{2026}
\arxivpagecount{16}
\arxivshortauthor{Abdelwahab, N.}

\copyrightyear{2026}

\arxivhistory{%
Originally received June 14, 2026;
Revised: June 17, 2026;
Accepted June 26, 2026;
published July 20, 2026%
}

\arxivpublicationline{%
JAcad(NY)15,1:3--18 (16 pages) --
Theoretical Computer Science -- permalink:%
}

\arxivpermalink{%
https://jacad.eth.link/rc_images/abdelwahab_3_18_15_1.pdf%
}

\FullResearchPaper

\title{Computational Extraction of Legal Causes via al-\d{S}abr wa
al-Taqsim}

\subtitle{A Set-Theoretic Formalization for Closed Fiqh Chapters}

\authorblock{%
Elnaser Abdelwahab *\\
makmad.org e.V.\\
Hannover (Germany)%
}

\dates{%
Received June 14, 2026;
Revised June 17, 2026;
Accepted June 26, 2026%
}


\correspondingauthor{%
*Corresponding author: elnaser@hotmail.com%
}

\shortauthor{Abdelwahab, N.}
\firstpage{1}

\begin{document}

\makejacadtitle

\begin{jacadabstract}
This paper presents a set-theoretic formalization of the classical usuli
method of al-\d{S}abr wa al-Taqsim (Examination and Division) for
extracting legal causes (`\textit{`ilal}) within closed chapters of
jurisprudence. A computational algorithm is introduced that extracts
minimal operational rules from a truth table of juristic verdicts. The
principal result is that, given a complete truth table for a closed chapter,
the algorithm computes the minimal structural generators of the ruling and
eliminates all logically redundant attributes. The resulting structures
constitute admissible candidate causes for subsequent juristic evaluation.
\end{jacadabstract}

\makejacadkeywords

\section{Introduction}
\label{sec:introduction}

Traditional inference in Islamic jurisprudence (fiqh) often proceeds through
case-based reasoning, textual interpretation, analogical extension, and forms
of induction that are not usually exhaustive. In addition to this, and due in
part to the influence of important fiqh scholars such as al-Ghazali and
Averroes, juristic rulings were also reached using classical deductive
Aristotelian and Avicennian logic \cite{StreetGermann2021}. Toward the end of
his life, al-Ghazali explicitly stressed the importance of deductive
logic-based reasoning, treating logic as a general introduction to the
sciences and as a condition for trustworthy scientific inquiry
\cite{GhazaliMustasfa}.

The present paper studies a restricted setting in which both analogical and
deductive reasoning mechanisms can be represented by a finite conceptual model.
For the purposes of the formal model, we consider a school-relative corpus of
legal source material and fiqh-ontological definitions as fixed. Relative to
such a fixed corpus, the paper asks whether a given chapter can be represented
by a finite set of legally operative concepts and a complete ruling function
over those concepts.

If the relevant source corpus, ontology vocabulary, and operative legal rules
are treated as fixed relative to a given school, a further question arises:
where does the apparent open-endedness of fiqh application come from? How can
a finite ontological structure govern indefinitely many concrete cases?

The answer is twofold. First, the open-ended application of a legal rule
emerges from the contexts in which that rule is applied to concrete world
situations. Second, the premises of a legal rule, although fixed within the
model, may apply to indefinitely many situations because the categories used in
the sources are abstract. For example, in the chapter of Tahara
(purification), many real-world situations may fall under the abstract concept
of an ``excuse'' (`\textit{`Udhr}), which serves as a premise for the rule of
Tayammum. While keeping the rule fixed and binding its premises to this
abstract notion, one can map indefinitely many concrete situations to it. We
call the structured set of such abstract legally operative notions for a given
fiqh chapter the \textbf{Ontological Concepts Tree (OCT)}.

This leads to the epistemological distinction between
Avicennian/Aristotelian Logic and Fiqhi Qiyas, which governs how the OCT
operates:
\begin{itemize}
    \item \textbf{Avicennian/Aristotelian Logic} relies on deductive syllogisms
    where the middle term is an ontological cause. The connection between premise
    and conclusion is one of strict necessity; if the cause is present, the effect
    must physically follow. The rule $A \Rightarrow B$ admits no exceptions
    (monotonic).

    \item \textbf{Fiqhi Qiyas} relies on analogical reasoning where the middle
    term is a teleological/legal cause (`\textit{`Illah}). The connection
    yields probabilistic knowledge and operates as a default rule; the Lawgiver
    may suspend the ruling due to a competing legal objective or preventative
    (\textit{Mani'}). Thus, $A \Rightarrow B$ is non-monotonic: the presence
    of a preventative $M$ can yield $A \land M \Rightarrow \neg B$.
\end{itemize}

To formalize this system computationally and logically, we must introduce
strict methodological assumptions. The most critical is the
\textbf{School-Relative Approach}. We do not treat the OCT as a subject of
philosophical dispute or semantic fuzziness. Instead, we adopt the
\textbf{Closed-World Assumption (CWA)}: the OCT is strictly bounded by the
syntactic and semantic definitions explicitly laid out in the primary reference
text (\textit{Matn} or \textit{Mukhtasar}) of a specific school of thought
(\textit{Madhhab}). By anchoring the OCT to a specific school's text, the
relevant meanings are treated as fixed relative to the selected school text.

A Fiqh chapter for which the school-relative OCT is constant is defined as a
\textbf{Closed Chapter}. Under this assumption, the combinatorics behind the
chapter, even when exhaustive, yield a finite, constant state space. Because
the number of OCT nodes ($n$) and the valid preventatives ($M$) are strictly
finite and bounded by the text, the total number of possible states is $2^n$.

This has two methodological consequences.
\begin{enumerate}
    \item \textbf{Shift to Complete Induction}: For a closed chapter, we can
    transition the mode of Fiqh inference from incomplete induction to complete
    induction. The finite state space allows for the exhaustive construction
    of a truth table over the OCT. The non-monotonicity of Fiqhi Qiyas can
    be represented as bounded defeasibility within this table (where rows
    containing preventatives override default rows).

    \item \textbf{Structural Comparison of Madhhabs}: By formalizing the OCTs
    and their respective inference rules on a chapter-by-chapter basis, we gain
    a rigorous, practical method for comparing different schools of thought.
    Instead of comparing only isolated rulings or relying solely on informal
	semantic comparison, we can computationally compare the structural logic of
	different Madhhabs: we can map the differences in their OCTs, analyze how
	different schools define preventatives ($M$), and identify where their
	truth tables diverge.
\end{enumerate}

In summary, the Closed Chapter Assumption does not claim to capture the
infinite metaphysical essence of the law; rather, it rigorously models the
internal logic of a Madhhab as it stands. This makes systematic structural
comparison across legal schools more explicit.

Some of the practical benefits of this approach are that it simplifies the
extraction of legal causes (`\textit{`ilal}), a task among the central tasks of
Usul al-Fiqh. This paper does not attempt to prove ontological
fiqh-causation. Rather, it formalizes the structural extraction of candidate
causes from a complete operational model represented as a truth table.

The present paper should also be read as a domain-specific continuation of
earlier work on logically complete ontology-level computation, where
ontological term systems, decision components, inductive components, deductive
components, and rational response components were proposed for relational
database systems \cite{Abdelwahab2020Patent}. Here, the same general idea is
specialized to closed fiqh chapters: the ontology is not a database ontology,
but a school-relative Ontological Concepts Tree (OCT), and the response
component is not a database query response, but the extraction of structurally
admissible candidate \textit{`ilal}.

The framework presented here is a modelling framework: it applies when the
relevant chapter can be projected, as an approximation, into a finite,
school-relative concept vocabulary and a complete ruling table over that
vocabulary.

\section{Foundational Assumptions}
\label{sec:foundational-assumptions}

\subsection{Conceptual Representation of Juristic Reality}

The present framework is based upon the observation that fiqh does not operate
directly upon physical objects, but rather upon legally operative concepts
abstracted from those objects. A juristic ruling is therefore determined not by
the physical identity of a case, but by the conceptual attributes recognized as
legally relevant within the chapter.

\begin{axiom}[Conceptual Representation Principle]
\label{ax:conceptual-representation}
Let $\Omega$ denote the set of all real-world cases and let
$V=\{v_1, v_2,\dots, v_n\}$ denote the set of legally operative concepts of a
fiqh chapter. There exists a mapping $\phi: \Omega \to \{0, 1\}^n$ such that
every ruling depends exclusively upon $\phi(\omega)$ for each
$\omega \in \Omega$, and not upon any additional physical characteristics of
the case.
\end{axiom}

The mapping $\phi$ may be viewed as an Ontological Concepts Tree (OCT) encoding
that transforms infinitely many real-world situations into a finite conceptual
representation.

\subsection{Conceptual Finiteness}

The existence of a finite concept vocabulary immediately induces a finite legal
state space.

\begin{theorem}[Conceptual Finiteness]
\label{thm:conceptual-finiteness}
If $V=\{v_1, v_2,\dots, v_n\}$ is finite, then the legal state space is finite.
\end{theorem}

\begin{proof}
By Axiom~\ref{ax:conceptual-representation}, every legal case is represented
by a valuation of the variables in $V$. Since each variable is binary, the
number of distinct represented legal states is at most
$|\{0, 1\}^n| = 2^n$. Because $n < \infty$, the quantity $2^n$ is finite.
Therefore the legal state space is finite.
\end{proof}

The significance of this theorem is that it separates the apparent infinity of
real-world instances from the finite conceptual structure upon which juristic
reasoning actually operates.

\subsection{Physical and Conceptual Closure}

The present framework does not require the physical world itself to be finite.
Instead it requires only the stability of the legally operative conceptual
vocabulary.

\begin{definition}[Physical Closure]
\label{def:physical-closure}
A domain is physically closed iff no new real-world entities can ever arise.
\end{definition}

\begin{definition}[Conceptual Closure]
\label{def:conceptual-closure}
A domain is conceptually closed iff no future legally operative concept exists
outside the concept set $V$.
\end{definition}

\begin{remark}
Physical closure is generally impossible. Conceptual closure may nevertheless
hold because infinitely many future cases can often be represented using a
fixed conceptual vocabulary.
\end{remark}

\subsection{Closed Chapters}

The notion of a Closed Chapter is the foundational assumption underlying the
remainder of the paper.

\begin{definition}[Closed Chapter]
\label{def:closed-chapter}
A fiqh chapter is said to be closed iff:
\begin{enumerate}
    \item its legally operative concepts are represented by a finite set
    $V=\{v_1, v_2,\dots, v_n\}$;

    \item every future case admits a representation
    $\phi(\omega) \in \{0, 1\}^n$;

    \item no future legally operative concept exists outside the set $V$.
\end{enumerate}
\end{definition}

A Closed Chapter therefore represents a conceptually complete operational model
of the chapter under investigation.

\subsection{Operational Closure and Conceptual Stability}

The purpose of closure is not to make a metaphysical claim about reality.
Rather, closure serves as an operational criterion governing whether the
conceptual vocabulary of the chapter remains sufficient.

\begin{principle}[Conceptual Stability]
\label{pr:conceptual-stability}
A chapter may be treated as operationally closed whenever all newly encountered
cases can be represented using the existing concept set $V$.
\end{principle}

Under this principle, the appearance of a new physical object or circumstance
does not invalidate closure provided that the object can be expressed through
already existing legally operative concepts.

\subsection{Preservation and Failure of Closure}

The following theorem characterizes when closure remains valid.

\begin{theorem}[Closure Preservation]
\label{thm:closure-preservation}
Let $C$ be a Closed Chapter with concept set $V$. Suppose a new case $\omega$
is encountered. If $\phi(\omega) \in \{0, 1\}^n$, then the chapter remains
closed.
\end{theorem}

\begin{proof}
Since $\phi(\omega)$ is expressible using the existing concept set $V$, no new
legally operative concept has been introduced. Therefore the conceptual
vocabulary remains unchanged. Hence closure is preserved.
\end{proof}

The converse theorem identifies the precise failure condition of the framework.

\begin{theorem}[Closure Failure]
\label{thm:closure-failure}
Let $C$ be a Closed Chapter with concept set $V$. Suppose a newly encountered
case requires a concept $x \notin V$. Then the Closed Chapter Assumption is
false.
\end{theorem}

\begin{proof}
The existence of $x \notin V$ contradicts Condition (3) of the Closed Chapter
definition. Therefore the chapter is not closed.
\end{proof}

This theorem provides a falsification criterion for the framework: whenever a
genuinely new legally operative concept is discovered, the chapter must be
re-opened and the model reconstructed.

\subsection{Complete Induction over Closed Chapters}

The principal consequence of closure is that it transforms juristic analysis
from incomplete induction into complete induction.

\begin{corollary}[Complete Induction over a Closed Chapter]
\label{cor:complete-induction}
If a chapter is closed, then exhaustive enumeration of all legal states is
possible.
\end{corollary}

\begin{proof}
By Theorem~\ref{thm:conceptual-finiteness}, the represented legal state space
contains at most $2^n$ distinct states. Since this quantity is finite, every state
can be enumerated. Therefore complete induction over the chapter is possible.
\end{proof}

Consequently, the logical foundation of the present framework is the following
chain:
\begin{align*}
&\text{Conceptual Representation} \Rightarrow \text{Conceptual Finiteness} \\
\Rightarrow\ &\text{Closed Chapter} \Rightarrow \text{Finite State Space} \\
\Rightarrow\ &\text{Complete Induction.}
\end{align*}

All subsequent algorithms and correctness proofs depend upon this foundational
sequence.

\subsection{Modular Decomposition and Hierarchical Truth Tables}

The Closed Chapter Assumption guarantees that each chapter's legally operative
concept set $V$ is finite and fixed. However, a complex chapter such as
\emph{Ḥajj} (pilgrimage) may involve dozens of potentially relevant concepts
if flattened into a single truth table. This would lead to a large $n$ and a
combinatorial state space of $2^n$ rows---still finite, but impractical to
construct manually.

This hierarchical separation between lower-level factual representation and
higher-level ontological/legal response is analogous to the earlier distinction
between database-level and ontology-level processing in logically complete
relational systems \cite{Abdelwahab2020Patent}.

Fortunately, fiqh chapters are \textbf{modular by design}. A higher-level
chapter does not operate directly on the atomic concepts of its sub-chapters;
instead, it relies on the \emph{outcome rulings} of those sub-chapters as
atomic inputs. For example:
\begin{itemize}
    \item The \textbf{Ḥajj} chapter uses the ruling of \textbf{Ṭahāra}
    (purification) as a single binary condition: ``Is the pilgrim in a valid
    state of purity (\emph{ṭāhir})?''

    \item The \textbf{Ṣalāh} chapter similarly uses ``ṭāhir'' as a condition,
    but Ṣalāh itself may be a sub-chapter of Ḥajj (e.g., the prayer of
    \emph{ṭawāf}).
\end{itemize}

This modular structure allows us to apply a \textbf{divide-and-conquer}
strategy:
\begin{enumerate}
    \item For each atomic closed sub-chapter (e.g., Ṭahāra), construct its
    complete truth table over its own small variable set $V_{\text{sub}}$.

    \item Compute the ruling function $h_{\text{sub}}$ from that table using
    the algorithm of Section~\ref{sec:algorithm}.

    \item Treat $h_{\text{sub}}$ as a \textbf{single binary variable} in the
    truth table of the higher-level chapter.
\end{enumerate}

The combinatorial complexity is thus bounded by the size of the largest
\emph{atomic} truth table, not by the total number of atomic concepts across
all sub-chapters.

\subsubsection{Concrete Example: Two-Level Hierarchy (Ṭahāra $\to$ Ṣalāh)}

Let the atomic chapter \textbf{Ṭahāra} (purification for prayer) have the
following binary variables (as in the Tayammum appendix):

\[
\begin{array}{c|l|c}
\text{Variable} & \text{Description} & \text{Values} \\
\hline
h & \text{Person has minor impurity (\emph{ḥadath})} &
1 = \text{impure}, 0 = \text{pure} \\
w & \text{Water available} &
1 = \text{available}, 0 = \text{lacking} \\
u & \text{Valid excuse (\emph{ʿudhr}) for tayammum} &
1 = \text{excused}, 0 = \text{not excused} \\
t & \text{Prayer time entered} &
1 = \text{entered}, 0 = \text{not entered} \\
d & \text{Valid dust (\emph{ṣaʿīd}) available} &
1 = \text{available}, 0 = \text{not available}
\end{array}
\]

The complete truth table for Ṭahāra has $2^5 = 32$ rows. Running the
algorithm on this table yields the minimal structural form:

\[
h_{\text{ṭahāra}} = \underbrace{(h \land t \land d)}_{\text{shurūṭ}} \;\land\;
\underbrace{(\neg w \;\lor\; u)}_{\text{candidate ʿillah}}
\]

Now consider the higher-level chapter \textbf{Ṣalāh} (the ritual prayer). The
Ṣalāh chapter does not need to expand the five internal variables of Ṭahāra.
Instead, it uses a single input $P = h_{\text{ṭahāra}}$ (1 if purification
permits prayer, 0 otherwise). Additional local variables for Ṣalāh might be:

\[
\begin{array}{c|l|c}
\text{Variable} & \text{Description} & \text{Values} \\
\hline
n & \text{Intention (\emph{niyyah}) present} &
1 = \text{present}, 0 = \text{absent} \\
q & \text{Facing the \emph{qibla}} &
1 = \text{facing}, 0 = \text{not facing} \\
c & \text{Covering the \emph{ʿawrah}} &
1 = \text{covered}, 0 = \text{not covered}
\end{array}
\]

The truth table for Ṣalāh therefore has $2^{1+3} = 16$ rows, not
$2^{5+3} = 256$. The modular decomposition reduces the state space by a
factor of 16 in this case.

\subsubsection{Recursive Application to Deeper Hierarchies}

The same pattern applies recursively. The output of the Ṣalāh chapter,
$h_{\text{ṣalāh}}$, becomes a single binary variable in the \textbf{Ḥajj}
truth table. Suppose Ḥajj adds two local variables (e.g., \emph{iḥrām}
properly entered, \emph{saʿy} between Ṣafā and Marwa completed). Then the Ḥajj
truth table has only $2^{1+2} = 8$ rows, as summarised below:

\[
\begin{array}{c|c|c|c|c}
\text{Level} & \text{Chapter} & \text{Local variables} &
\text{Sub-chapter inputs} & \text{Total rows} \\
\hline
1 & \text{Ṭahāra} & 5 & 0 & 2^5 = 32 \\
2 & \text{Ṣalāh} & 3 & 1 & 2^{3+1} = 16 \\
3 & \text{Ḥajj} & 2 & 1 & 2^{2+1} = 8
\end{array}
\]

\textbf{Total computational work} = processing 32 rows (Ṭahāra) + 16 rows
(Ṣalāh) + 8 rows (Ḥajj) = 56 rows. In contrast, a flattened single-table
approach with all $5+3+2 = 10$ atomic variables would require
$2^{10} = 1024$ rows and a search space of $3^{10} = 59,049$
combinations---still feasible for a computer, but significantly larger and
harder for a human jurist to populate.

\subsubsection{Implications for the Framework}

\begin{enumerate}
    \item \textbf{Tractability:} The effective $n$ at any level of the
    hierarchy is the number of \textbf{sub-chapter outputs} plus the few
    \textbf{local variables} at that level. Since each sub-chapter is itself
    closed and small, the overall system remains computationally practical,
    and the $O(1)$ complexity result applies to each atomic table individually.

    \item \textbf{Reusability:} Once an atomic chapter (e.g., Ṭahāra) is fully
    tabulated and its minimal structure extracted, the resulting binary
    variable can be reused across all higher chapters (Ṣalāh, Ḥajj, Ṣawm,
    etc.). This eliminates redundant work.

    \item \textbf{Comparative madhhab analysis:} Modular decomposition allows
    researchers to compare madhhabs at each level independently. For example,
    two schools may agree on the Ṭahāra truth table but differ on the Ṣalāh
    truth table; the algorithm can identify where their rulings diverge.

    \item \textbf{Practical construction of truth tables:} A team of jurists
    can populate atomic tables once, validate them against primary sources, and
    then compose them into larger tables without revisiting lower-level
    details. This turns the construction of chapter-level models into a finite
    and auditable formalization task, while leaving interpretive and doctrinal
    validation to qualified juristic scholarship.
\end{enumerate}

\subsubsection{Remark on Non-Modular Dependencies}

Occasionally, a higher chapter may depend on a sub-chapter's
\emph{internal variables} rather than its final ruling---for example, when the
ruling of Ṣalāh depends not only on whether Ṭahāra is valid, but on \emph{why}
it is valid (e.g., tayammum vs. water). In such cases, the sub-chapter's output
must be refined into multiple binary flags (e.g., $P_{\text{water}}$ and
$P_{\text{tayammum}}$). The modular decomposition remains valid, but the
interface between chapters becomes a small vector of outputs instead of a
single bit. The number of such flags is still bounded by the sub-chapter's
number of minimal positive rules, which is itself bounded by the small
$n_{\text{sub}}$. Therefore, the combinatorial explosion remains contained.

\section{Algorithm}
\label{sec:algorithm}

We model the jurist's decision space as a Boolean truth table over the variable
set $V$ of a closed chapter. Let $V=\{v_1, v_2,\dots, v_n\}$ be the set of
binary decision variables and $h$ be the binary ruling variable. The algorithm
operates in three distinct phases: Rule Extraction, Rule Shortening, and Usul
Deduction.

\begin{algorithm}[H]
\caption{ExtractRules(U, V)}
\label{alg:extract-rules}
\begin{algorithmic}[1]
\Require Universe of rows $U$, Variable set $V$
\Ensure Set of valid rules $R$
\State $R \leftarrow \emptyset$
\For{each subset $P \subset V$ of size $j$}
    \For{each valuation $w$ of variables in $P$}
        \State $H_w \leftarrow \{h(s) \mid s \in U, s \text{ matches } w\}$
        \If{$|H_w| = 1$}
            \State $R \leftarrow R \cup \{(w \Rightarrow h_w)\}$
        \EndIf
    \EndFor
\EndFor
\State \Return $R$
\end{algorithmic}
\end{algorithm}

\begin{algorithm}[H]
\caption{ShortenRules(R)}
\label{alg:shorten-rules}
\begin{algorithmic}[1]
\Require Set of valid rules $R$
\Ensure Set of minimal rules $M$
\State $M \leftarrow R$
\For{each $R_j \in R$}
    \For{each $R_k \in M$}
        \If{$R_k \subset R_j$ and $R_k$ implies same ruling as $R_j$}
            \State Remove $R_j$ from $M$
        \EndIf
    \EndFor
\EndFor
\State \Return $M$
\end{algorithmic}
\end{algorithm}

\begin{algorithm}[H]
\caption{DeduceUsul(M+, M-)}
\label{alg:deduce-usul}
\begin{algorithmic}[1]
\Require Positive minimal rules $M^+$, Negative minimal rules $M^-$
\Ensure Shurut $S$, `Illah set $I$, Mawani' $W$
\State $S \leftarrow \bigcap_{R \in M^+} R$
\State $I \leftarrow \emptyset$
\For{each $R_i \in M^+$}
    \State $I_i \leftarrow R_i \setminus S$
    \If{$I_i \neq \emptyset$}
        \State $I \leftarrow I \cup \{I_i\}$
    \Else \Comment{Heuristic}
        \State $I_{\text{trigger}} \leftarrow S \setminus
        \{v \in S \mid v \text{ is framework variable}\}$
        \State $I \leftarrow I \cup \{I_{\text{trigger}}\}$
        \State $S \leftarrow S \setminus I_{\text{trigger}}$
    \EndIf
\EndFor
\State $W \leftarrow \emptyset$
\State $I_{\text{flat}} \leftarrow \bigcup_{I_k \in I} I_k$
\Comment{Flatten set of candidate causes into literal set}
\For{each $R_j \in M^-$}
    \If{$R_j \not\subseteq \text{Inv}(S) \cup
    \text{Inv}(I_{\text{flat}})$}
        \State $W \leftarrow W \cup \{R_j\}$
    \EndIf
\EndFor
\State \Return $S, I, W$
\end{algorithmic}
\end{algorithm}

Let $U$ denote the universe of rows of the truth table.

\subsection{Rule Extraction}

Generate all rules $R \Rightarrow h$ that hold universally within $U$.

\subsection{Rule Minimization}

Remove every rule that possesses a strictly smaller valid subrule. The
remaining set is denoted $M$.

\subsection{Usuli Deduction}

Let $M^+$ be the positive minimal rules and $M^-$ the negative minimal rules.
Compute:
\[
S = \bigcap_{R \in M^+} R
\]
Then for each $R_i \in M^+$ compute $I_i = R_i \setminus S$. The set $S$
represents candidate conditions, while the $I_i$ represent candidate causes.

\section{Formal Axiomatization}
\label{sec:formal-axiomatization}

\begin{axiom}[Shart]
\label{ax:shart}
A variable is a Shart iff it is necessary for the ruling.
\end{axiom}

\begin{axiom}[Structural Causal Correspondence]
\label{ax:structural-causal-correspondence}
Within a Closed Chapter, every minimal positive rule corresponds to one
admissible causal pathway in the jurist's operational model.
\end{axiom}

The role of the algorithm is to identify minimal structural generators of the
ruling rather than independently establish ontological fiqh-causation.

\begin{definition}[Candidate `Illah]
\label{def:candidate-illah}
For a minimal positive rule $R_i$, the Candidate `Illah is $I_i = R_i
\setminus S$.
\end{definition}

\begin{definition}[Inverse Literal]
\label{def:inverse-literal}
For a literal $x$, $\text{Inv}(x) = \neg x$. For a set $A$,
$\text{Inv}(A) = \{\neg x : x \in A\}$.
\end{definition}

\section{Correctness Results}
\label{sec:correctness-results}

\begin{theorem}[Correctness of Shart Extraction]
\label{thm:shart-extraction}
The algorithm correctly identifies all and only the Shurut.
\end{theorem}

\begin{proof}
The algorithm computes $S = \bigcap_{R \in M^+} R$. A variable belongs to $S$
iff it appears in every positive minimal rule. Therefore it is necessary for
every occurrence of the ruling. Conversely, every necessary condition must
appear in every positive rule. Hence $S$ is exactly the set of Shurut.
\end{proof}

\begin{theorem}[Correctness of Branch-Factor Extraction]
\label{thm:branch-factor-extraction}
For every minimal rule $R_i$, the algorithm correctly extracts
$I_i = R_i \setminus S$.
\end{theorem}

\begin{proof}
Since $S = \bigcap_{R \in M^+} R$, $S$ contains precisely the universal
components of all positive rules. Removing $S$ from $R_i$ leaves exactly the
branch-specific component. Thus $I_i = R_i \setminus S$ is correctly extracted.
\end{proof}

\begin{corollary}
\label{cor:candidate-illah}
Under the Structural Causal Correspondence Axiom, each $I_i$ is an admissible
Candidate `Illah.
\end{corollary}

\begin{theorem}[Structural Classification of Negative Rules (Mani')]
\label{thm:negative-rule-classification}
Let $I_{\text{flat}} = \bigcup_i I_i$. A negative rule $R_j$ is classified as
absence-based if
\[
R_j \subseteq \text{Inv}(S) \cup \text{Inv}(I_{\text{flat}}).
\]
If this inclusion fails, then $R_j$ contains at least one literal not
accounted for by the absence of an extracted condition or candidate cause.
Such a literal is classified by the formal model as a candidate Mani'.
\end{theorem}

\begin{proof}
If $R_j \subseteq \text{Inv}(S) \cup \text{Inv}(I_{\text{flat}})$, then the
negative rule is explained within the formal model by the absence of an
extracted condition or candidate cause. If this inclusion fails, then $R_j$
contains at least one literal outside
$\text{Inv}(S) \cup \text{Inv}(I_{\text{flat}})$. Such a literal is not
accounted for by mere absence of the extracted Shurut or candidate `Ilal.
Therefore the formal model classifies it as a candidate Mani'. \qedhere
\end{proof}

\begin{heuristic}[Recovery of the Candidate `Illah Qasirah]
\label{heur:illah-qasirah}
When $|M^+| = 1$, the intersection operation yields $S = R_1$, and therefore
$R_1 \setminus S = \emptyset$. In this special case, external semantic
knowledge supplied by the jurist is required to distinguish framework
variables from triggering variables. This procedure is heuristic and is not
derivable from the truth table alone.
\end{heuristic}

\section{Legitimacy of the Candidate `Illah}
\label{sec:legitimacy-candidate-illah}

\begin{definition}[False Flag]
\label{def:false-flag}
A false flag is a variable that co-occurs with the ruling but is logically
redundant.
\end{definition}

\begin{theorem}[Elimination of Logically Redundant Literals]
\label{thm:redundant-literals}
No Candidate `Illah extracted from a minimal rule can contain a false flag.
\end{theorem}

\begin{proof}
Assume a Candidate `Illah contains a redundant attribute. Then there exists a
smaller rule $R' \subset R_i$ that still implies the ruling. This contradicts
the minimality of $R_i$. Therefore no extracted Candidate `Illah contains a
false flag.
\end{proof}

\section{Co-Extensiveness}
\label{sec:co-extensiveness}

\begin{definition}[Global Candidate `Illah]
\label{def:global-candidate-illah}
Let $I_1, \ldots, I_m$ be all Candidate `Ilal extracted from the positive
minimal rules $R_1,\ldots,R_m \in M^+$. Define $I^* = I_1 \lor I_2 \lor
\cdots \lor I_m$. The formula $I^*$ is called the Global Candidate `Illah.
\end{definition}

\begin{theorem}[Structural Decomposition and Co-Extensiveness]
\label{thm:structural-decomposition}
Assume that $M^+ = \{R_1,\ldots,R_m\}$ is the complete set of positive minimal
rules covering the positive region of the ruling function $h$. Let
$S = \bigcap_{R_i \in M^+} R_i$ denote the extracted Shurut, and let
$I_i = R_i \setminus S$ be the branch-specific component of each positive
minimal rule. If $I^* = I_1 \lor I_2 \lor \cdots \lor I_m$, then the ruling
function admits the structural decomposition $h = S \land I^*$. Consequently,
$h \iff (S \land I^*)$.
\end{theorem}

\begin{proof}
By the definition of $S$ as the intersection of all positive minimal rules, we
have $S \subseteq R_i$ for every $R_i \in M^+$. By the definition
$I_i = R_i \setminus S$, each positive minimal rule decomposes as
$R_i = S \land I_i$. Since $M^+$ is assumed to be the complete set of positive
minimal rules covering the positive region of $h$, the ruling function is
represented by $h = R_1 \lor R_2 \lor \cdots \lor R_m$. Substituting
$R_i = S \land I_i$ gives $h = (S \land I_1) \lor (S \land I_2) \lor \cdots
\lor (S \land I_m)$. By distributivity of conjunction over disjunction, this
becomes $h = S \land (I_1 \lor I_2 \lor \cdots \lor I_m)$. By
Definition~\ref{def:global-candidate-illah}, we have $I^* = I_1 \lor I_2
\lor \cdots \lor I_m$. Therefore $h = S \land I^*$. Hence
$h \iff (S \land I^*)$. \qedhere
\end{proof}

\section{Structural Soundness}
\label{sec:structural-soundness}

\begin{theorem}[Structural Soundness]
\label{thm:structural-soundness}
Given: (1) a Closed Chapter, (2) a complete variable set $V$, and (3) a
complete truth table, the algorithm computes exactly the subset-minimal
implicating partial valuations for the ruling function $h$.
\end{theorem}

\begin{proof}
Algorithm~\ref{alg:extract-rules} enumerates all valid implicating rules.
Algorithm~\ref{alg:shorten-rules} removes every non-minimal rule. The
remaining set $M$ is exactly the set of subset-minimal implicating partial
valuations for $h$. All subsequent classifications are computed from this set.
\end{proof}

\section{Semantic Boundary}
\label{sec:semantic-boundary}

The algorithm establishes structural legitimacy but not semantic suitability
(\textit{munasabah}). If irrelevant variables are supplied, the algorithm will
correctly identify their structural role within the truth table, but cannot
determine whether they reflect the objectives of the Lawgiver. That
responsibility remains with the faqih.

\section{Computational Complexity and Tractability}
\label{sec:complexity}

A rigorous analysis of the algorithm's complexity requires distinguishing
between the general computational complexity of exhaustive rule extraction and
the resource constraints imposed by the Closed Chapter Assumption. We
demonstrate that while the algorithm performs a task that is classically
exponential in the number of variables, a fixed closed chapter yields a
bounded finite computation whose upper bound depends only on the size of that
chapter's concept set.

\begin{theorem}[Bounded Resource Consumption for a Fixed Closed Chapter]
\label{thm:bounded-resource}
For any fixed closed chapter with a fixed concept set of size \(k\), the
algorithmic pipeline has a constant upper bound depending only on \(k\).
\end{theorem}

\begin{remark}
This is not a polynomial-time claim in an unbounded variable parameter \(n\);
rather, it is a bounded-domain claim relative to a fixed closed chapter.
\end{remark}

\begin{proof}
Unlike the Quine-McCluskey algorithm, which iteratively combines adjacent
minterms bottom-up, Algorithm~\ref{alg:extract-rules} operates top-down via
exhaustive subset-valuation projection. The mechanism proceeds as follows:
\begin{enumerate}
    \item Generate all possible subsets $s \subset V$ of the variable set.

	\item For a given subset $s$ of size $j$, generate all $2^j$ possible
	valuations $w$ of the variables in $s$.

    \item Project the valuation $w$ onto the truth table $U$. If all rows
    matching the valuation $w$ yield a uniform ruling (i.e., the projection
    reduces to a single value), then deduce the rule $r: (w \Rightarrow h_w)$.

    \item Accumulate all such valid rules into the rule set $R$.
\end{enumerate}

Algorithm~\ref{alg:shorten-rules} subsequently minimizes this set by removing any
rule strictly subsumed by a smaller valid rule. The resulting set $M$ is
mathematically equivalent to the set of prime implicants of the ruling function
\cite{Quine1952, McCluskey1956}, though derived via a different algorithmic
pathway.

To determine the complexity, we evaluate the number of subset-valuation
combinations Algorithm~\ref{alg:extract-rules} must project. For a variable
set of size $n$, the total number of valuations evaluated across all subsets is:
\begin{equation*}
\sum_{j=0}^{n} \binom{n}{j} 2^j = (1+2)^n = 3^n.
\end{equation*}
This operation is classically exponential with respect to the number of
variables, $n$. If $n$ were unbounded, the algorithm would require exponential
resources.

However, by the Closed Chapter Assumption, the legally operative concept set
$V$ is finite and fixed for a given Madhhab. Therefore, the number of variables
is strictly bounded such that $n = k$ for some constant $k > 0$.

Because $n$ is a fixed constant $k$ for the chapter under investigation, we can
establish strict constant upper bounds for all parameters in the algorithm:
\begin{enumerate}
    \item The number of subset-valuation combinations evaluated by
    Algorithm~\ref{alg:extract-rules} is exactly $3^k$. For each combination,
    projecting and checking for a uniform ruling requires at most $O(n)$
    operations. Thus, the runtime of Algorithm~\ref{alg:extract-rules} is
    $O(3^k \cdot k) = O(1)$.

    \item The number of rules generated by Algorithm~\ref{alg:extract-rules} is
    $m = |R|$. Since Algorithm~\ref{alg:extract-rules} generates at most one rule
    per evaluated valuation, $m \le 3^k$.

    \item The number of minimal rules output by Algorithm~\ref{alg:shorten-rules}
    is $p = |M^+| + |M^-|$. Since Algorithm~\ref{alg:shorten-rules} only removes
    rules from $R$ to form $M$, it follows that $p \le m \le 3^k$.
\end{enumerate}

Substituting these bounded parameters into the complexity of the remaining
algorithms:
\begin{itemize}
    \item Algorithm~\ref{alg:shorten-rules} runs in
    $O(m^2 \cdot k) \le O((3^k)^2 \cdot k) = O(9^k \cdot k) = O(1)$.

    \item Algorithm~\ref{alg:deduce-usul} runs in
    $O(p \cdot k) \le O(3^k \cdot k) = O(1)$.
\end{itemize}

Since all parameters governing the loops and operations across all three
algorithms depend exclusively on the constant $k$, relative to variation outside
the fixed chapter, the total time complexity of the pipeline is bounded by a
constant depending only on $k$. Thus, for a fixed closed chapter, the top-down
projection of all variable subsets is bounded by a constant depending only on
that chapter's fixed concept set. \qedhere
\end{proof}

\begin{remark}[Relation to Quine-McCluskey Minimization]
While Algorithm~\ref{alg:extract-rules} and Algorithm~\ref{alg:shorten-rules}
jointly compute the set of prime implicants---mathematically equivalent to the
output of the Quine-McCluskey algorithm---the top-down subset-valuation projection
method is better suited for the formalization of Fiqh. Quine-McCluskey
is not by itself tailored to the present juristic task for three reasons:
\begin{enumerate}
    \item \textbf{Multi-Valued Rulings (Aḥkām):} Quine-McCluskey is strictly
    designed for binary Boolean functions ($h \in \{0, 1\}$). In Fiqh, the
    ruling variable $h$ often takes on multiple discrete values (e.g.,
    \textit{Wājib}, \textit{Mandūb}, \textit{Mubāḥ}, \textit{Makrūh},
    \textit{Ḥarām}). Algorithm~\ref{alg:extract-rules} naturally extracts rules
    for any uniform ruling category by simply checking if the projection $H_w$
    contains a single value, regardless of the cardinality of the ruling set.

    \item \textbf{Symmetric Extraction of Mawāni' (Preventatives):}
    Quine-McCluskey is optimized to minimize the positive cover (the
    ``ON-set'') using ``don't-care'' states to simplify the circuit. Fiqh
    admits no ``don't-care'' states; the negative space must be explicitly
    analyzed to extract \textit{Mawāni'} (independent preventatives).
    Algorithm~\ref{alg:extract-rules} symmetrically extracts both positive ($M^+$)
    and negative ($M^-$) minimal rules, which is a strict prerequisite for the
    \textit{Usulī} deduction of \textit{Mawāni'} in Algorithm~\ref{alg:deduce-usul}.

    \item \textbf{Preservation of Causal Pathways:} Quine-McCluskey combines
    minterms bottom-up, which optimizes circuit logic but obscures the distinct
    causal pathways (\textit{ṭuruq al-ʿillah}). The top-down projection of
    Algorithm~\ref{alg:extract-rules} preserves the disjunctive normal form
    structure necessary to separate the \textit{Shurūṭ} ($S$) from the
    branch-specific \textit{'Ilal} ($I_i$), directly mapping the output to the
    epistemic structure of \textit{al-Ṣabr wa al-Taqsim}.
\end{enumerate}
\end{remark}

\begin{remark}
The significance of the preceding bounded-resource theorem is that it cleanly
separates the theoretical limits of Boolean logic from the operational reality
of Fiqh. In digital circuit design, $n$ can grow infinitely (e.g., $n > 50$),
rendering exhaustive subset projection computationally intractable. In Fiqh,
the Closed Chapter Assumption ensures $n$ remains small and strictly bounded by
the text of the Madhhab (typically $n \le 10$ for a single legal chapter).
The epistemological closure of the chapter mathematically guarantees the
computational tractability of the algorithm.
\end{remark}

\section{Conclusion}
\label{sec:conclusion}

Given a complete truth table for a Closed Chapter, the algorithm computes the
minimal structural generators of the ruling, eliminates all logically
redundant attributes, and produces formally admissible candidate
\textit{`ilal}. This is illustrated in the appendix through a reduced version
of the OCT for the Tahara chapter of the Shafi'i school. The algorithm
therefore constitutes a rigorous computational analogue of al-\d{S}abr wa al-Taqsim,
while leaving semantic validation to juristic expertise.

A natural question concerns the range of applicability of the present
framework. The framework does not claim that every chapter of fiqh is
necessarily a Closed Chapter, nor that the complete conceptual vocabulary of
every legal domain has already been discovered. Rather, it establishes that
whenever a chapter can be represented as a conceptually closed operational
model---namely, when all legally operative concepts have been identified and
all future cases can be expressed using that fixed conceptual vocabulary---the
method of \textit{al-\d{S}abr wa al-Taqsim} admits an exact computational
formalization through complete induction over a finite state space. The
validity of the mathematical results therefore depends not upon the universal
truth of the Closed Chapter Assumption, but upon its satisfaction within the
specific chapter under investigation. Consequently, the framework should be
understood as a conditional theory: if conceptual closure holds, then the
extraction of structurally admissible candidate causes reduces to the
computation of subset-minimal implicating partial valuations of the corresponding
ruling function.

\appendix



\section*{Supplementary material}

بسم الله الرحمن الرحيم

# تحليل أصولي حاسوبي لمسألة: التيمم (المذهب الشافعي)

المستخرج آلياً من جدول الحقائق المنطقي

---

## القسم الأول: جدول تصور المسائل (جدول الحقائق)

تم ضبط عرض الأعمدة وحجم الخط لضمان احتواء الجدول بالكامل ضمن هوامش الصفحة دون تجاوزها.

| المحدث | الماء | العذر | الوقت | الأداة | الحكم |
|---|---|---|---|---|---|
| المكلف طاهر | واجد الماء | غير معذور | لم يدخل الوقت | الأداة غير صالحة | غير جائز |
| المكلف طاهر | واجد الماء | غير معذور | لم يدخل الوقت | الأداة صالحة | غير جائز |
| المكلف طاهر | واجد الماء | غير معذور | دخل الوقت | الأداة غير صالحة | غير جائز |
| المكلف طاهر | واجد الماء | غير معذور | دخل الوقت | الأداة صالحة | غير جائز |
| المكلف طاهر | واجد الماء | معذور | لم يدخل الوقت | الأداة غير صالحة | غير جائز |
| المكلف طاهر | واجد الماء | معذور | لم يدخل الوقت | الأداة صالحة | غير جائز |
| المكلف طاهر | واجد الماء | معذور | دخل الوقت | الأداة غير صالحة | غير جائز |
| المكلف طاهر | واجد الماء | معذور | دخل الوقت | الأداة صالحة | غير جائز |
| المكلف طاهر | فاقد الماء | غير معذور | لم يدخل الوقت | الأداة غير صالحة | غير جائز |
| المكلف طاهر | فاقد الماء | غير معذور | لم يدخل الوقت | الأداة صالحة | غير جائز |
| المكلف طاهر | فاقد الماء | غير معذور | دخل الوقت | الأداة غير صالحة | غير جائز |
| المكلف طاهر | فاقد الماء | غير معذور | دخل الوقت | الأداة صالحة | غير جائز |
| المكلف طاهر | فاقد الماء | معذور | لم يدخل الوقت | الأداة غير صالحة | غير جائز |
| المكلف طاهر | فاقد الماء | معذور | لم يدخل الوقت | الأداة صالحة | غير جائز |
| المكلف طاهر | فاقد الماء | معذور | دخل الوقت | الأداة غير صالحة | غير جائز |
| المكلف طاهر | فاقد الماء | معذور | دخل الوقت | الأداة صالحة | غير جائز |
| المكلف محدث | واجد الماء | غير معذور | لم يدخل الوقت | الأداة غير صالحة | غير جائز |
| المكلف محدث | واجد الماء | غير معذور | لم يدخل الوقت | الأداة صالحة | غير جائز |
| المكلف محدث | واجد الماء | غير معذور | دخل الوقت | الأداة غير صالحة | غير جائز |
| المكلف محدث | واجد الماء | غير معذور | دخل الوقت | الأداة صالحة | غير جائز |
| المكلف محدث | واجد الماء | معذور | لم يدخل الوقت | الأداة غير صالحة | غير جائز |
| المكلف محدث | واجد الماء | معذور | لم يدخل الوقت | الأداة صالحة | غير جائز |
| المكلف محدث | واجد الماء | معذور | دخل الوقت | الأداة غير صالحة | غير جائز |
| المكلف محدث | واجد الماء | معذور | دخل الوقت | الأداة صالحة | **جائز** |
| المكلف محدث | فاقد الماء | غير معذور | لم يدخل الوقت | الأداة غير صالحة | غير جائز |
| المكلف محدث | فاقد الماء | غير معذور | لم يدخل الوقت | الأداة صالحة | غير جائز |
| المكلف محدث | فاقد الماء | غير معذور | دخل الوقت | الأداة غير صالحة | غير جائز |
| المكلف محدث | فاقد الماء | غير معذور | دخل الوقت | الأداة صالحة | **جائز** |

| المحدث | الماء | العذر | الوقت | الأداة | الحكم |
|---|---|---|---|---|---|
| المكلف محدث | فاقد الماء | معذور | لم يدخل الوقت | الأداة غير صالحة | غير جائز |
| المكلف محدث | فاقد الماء | معذور | لم يدخل الوقت | الأداة صالحة | غير جائز |
| المكلف محدث | فاقد الماء | معذور | دخل الوقت | الأداة غير صالحة | غير جائز |
| المكلف محدث | فاقد الماء | معذور | دخل الوقت | الأداة صالحة | **جائز** |

---

## القسم الثاني: القواعد الفقهية المجملة (القواعد المستخلصة)

### أولاً: قواعد الإثبات (متى يكون الحكم جائزاً/مثبتاً):

1. ( المكلف محدث و فاقد الماء و دخل الوقت و الأداة صالحة ) يقتضي حكماً: جائز
2. ( المكلف محدث و معذور و دخل الوقت و الأداة صالحة ) يقتضي حكماً: جائز

### ثانياً: قواعد النفي (متى يكون الحكم غير جائز/منفياً):

1. ( المكلف طاهر ) يقتضي حكماً: غير جائز
2. ( لم يدخل الوقت ) يقتضي حكماً: غير جائز
3. ( الأداة غير صالحة ) يقتضي حكماً: غير جائز
4. ( واجد الماء و غير معذور ) يقتضي حكماً: غير جائز

---

## القسم الثالث: التحليل الأصولي (استنباط الشروط والعلة والموانع)

1. **الشروط (مقومات الحكم):**
   - المكلف محدث
   - الأداة صالحة
   - دخل الوقت

   [وجه الاستنباط]: لأنها تلازم الحكم في جميع صوره المثبتة.

2. **العلة (المناط المؤثر في الحكم):**
   تعددت طرق العلة، وهي:
   - فاقد الماء
   - معذور

   [وجه الاستنباط]: لأنها الوصف المؤثر الذي تتعدد طرقه وتوجب الحكم.

3. **انتفاء الشروط والعلل (أسباب عدم الحكم):**

- انتفاء شرط: المكلف طاهر
- انتفاء شرط: لم يدخل الوقت
- انتفاء شرط: الأداة غير صالحة
- انتفاء العلة: غير معذور و واجد الماء

[وجه الاستنباط]: لأن انتفاء الشرط أو العلة غياب لمقومات الحكم لا مانع مستقل.

4. **الموانع المستقلة (ما يمنع الحكم مع تحقق الشروط والعلل):**

- لا يوجد مانع مستقل في هذه المسألة.

[وجه الاستنباط]: لا يوجد مانع مستقل عن انتفاء الشروط أو العلة.

**انتهى التحليل برب التوفيق**

## Statements and Declarations

### Competing Interests

The authors have no competing interests to declare.



\section*{Statements and Declarations}

\subsection*{Competing Interests}

The authors have no competing interests to declare.

\end{document}